\documentclass[accepted]{uai2021} 

\usepackage[american]{babel}

\usepackage[round]{natbib} 
\usepackage{mathtools} 
\usepackage{booktabs} 
\usepackage{tikz} 
\usepackage[utf8]{inputenc}
\usepackage{amssymb, xparse, mathtools, bbm, amsthm}
\usepackage{algpseudocode}
\usepackage{algorithm}
\usepackage{subcaption}
\usepackage{booktabs}
\usepackage{xcolor}
\usepackage{xr}

\newtheorem{theorem}{Theorem}
\newtheorem{lemma}{Lemma}
\newtheorem{observation}{Observation}

\theoremstyle{definition}
\newtheorem{definition}{Definition}

\makeatletter
\newcommand*{\addFileDependency}[1]{
  \typeout{(#1)}
  \@addtofilelist{#1}
  \IfFileExists{#1}{}{\typeout{No file #1.}}
}
\makeatother

\RenewDocumentCommand\Pr{sO{}r()}{%
  \operatorname{\mathbb{P}}%
  \begingroup
  \IfBooleanTF{#1}
    {\PrInn*{#3}}
    {\PrInn[#2]{#3}}%
  \endgroup
}

\DeclarePairedDelimiterX\PrInn[1][]{%
  \activatebar
  #1%
}

\newcommand{\activatebar}{%
  \begingroup\lccode`~=`|
  \lowercase{\endgroup\def~}{\,\delimsize\vert\,}%
  \mathcode`|=\string"8000
}

\DeclareMathOperator{\sign}{sign}

\DeclareMathAlphabet{\mathcal}{OMS}{cmsy}{m}{n}

\title{Addressing Fairness in Classification with a Model-Agnostic Multi-Objective Algorithm}

\author[1]{\href{mailto:Kirtan Padh <kirtan.701@gmail.com>?Subject=Your UAI 2021 paper}{Kirtan~Padh}{}\thanks{Most of the work was done while at EPFL and Swisscom.}} 
\author[2]{Diego Antognini}
\author[3]{Emma Lejal-Glaude}
\author[2]{Boi Faltings}
\author[3]{Claudiu Musat}
\affil[1]{Helmholtz AI, Germany}
\affil[2]{Ecole Polytechnique Fédérale de Lausanne, Switzerland}
\affil[3]{Swisscom, Switzerland}

\begin{document}
\maketitle

\begin{abstract}
  The goal of fairness in classification is to learn a classifier that does not discriminate against groups of individuals based on \textit{sensitive attributes}, such as race and gender. One approach to designing fair algorithms is to use relaxations of fairness notions as regularization terms or in a constrained optimization problem. We observe that the hyperbolic tangent function can approximate the indicator function. We leverage this property to define a differentiable relaxation that approximates fairness notions provably better than existing relaxations. In addition, we propose a model-agnostic multi-objective architecture that can simultaneously optimize for multiple fairness notions and multiple sensitive attributes and supports all statistical parity-based notions of fairness. We use our relaxation with the multi-objective architecture to learn fair classifiers. Experiments on public datasets show that our method suffers a significantly lower loss of accuracy than current debiasing algorithms relative to the unconstrained model.
\end{abstract}

\section{Introduction}
\textbf{Machine learning is omnipresent.} Machine learning systems have become ubiquitous in our daily lives and society. They are being adopted into an increasing variety of applications at an accelerating pace, including high-impact domains such as healthcare, job hiring, education, and criminal justice, among others~\citep{barocas-hardt-narayanan}. Despite this, questions remain on the ethical soundness of many such algorithms, as AI/ML systems have often been demonstrated to have unintentional and undesirable biases against \textit{sensitive attributes} such as age, gender, and race.

\textbf{Automated predictions can be biased.}  We consider an algorithm as biased or discriminatory when it does not satisfy a preconceived notion of equality with respect to one or more sensitive attributes. The COMPAS score~\citep{propublica16}, used in courts in the U.S. to predict the probability of recidivism, is one of the most well-known examples of discrimination by algorithms~\citep{propublica16}. Among the defendants who do not re-offend, the algorithm predicts black defendants to be higher risk at a much higher rate than white defendants. This can, in turn, lead to a further exacerbation of systemic bias through a negative feedback loop where the results of the algorithm bias the data even further, reflecting the bias even more in the next round of predictions.

\textbf{The bias can increase over time.}  A similar bias, which consists of reinforcing existing beliefs, is also present on social media: the filter bubble~\citep{pariser11}. The system recommends content that we tend to agree with, further reinforcing our views and putting us in an ``echo chamber" with other users with similar views, leading to polarization with users with opposing views. This is believed to have heavily influenced the $2016$ U.S. presidential elections~\citep{baer16}, and it is the kind of bias that can, over time, change the structure of society. Just as ever-present machine learning algorithms are in society, so is the unintentional algorithmic bias arising from such applications, thus making it critical to study fairness in machine learning.

\textbf{Debiasing approaches can be divided into three main categories}. Firstly, we have \textit{pre-processing algorithms}, where the data is processed before training to rid it of bias with the expectation that the classifier learned on the modified data would be fair~\citep{kamiran2012data,sattigeri2019fairness, calmon2017optimized}. Secondly, we have \textit{in-processing algorithms} that propose changes at training time, often in the form of minor changes to existing architectures, or entirely different algorithms~\citep{celis2019classification,lohaus20,zafar2017linear}. One approach to in-processing is to define relaxations of fairness notions and solve a constrained optimization problem or use the relaxations as regularization terms. Lastly, there are the \textit{post-processing algorithms} that filter the output of the classifier to ensure fairness~\citep{hardt2016equality, chierichetti2017fair,chzhen2019leveraging}.

\textbf{Debiasing is a naturally multi-objective problem.}  Most real-world applications have multiple sensitive attributes. We might want to satisfy different fairness notions for each attribute or several notions for a single attribute, making debiasing a naturally multi-objective problem. However, the research on multi-objective approaches to fairness is very sparse: most methods are specialized towards a specific fairness notion and only apply to a single attribute. Moreover, many fairness relaxations do not approximate the true fairness value well~\citep{lohaus20}.

\textbf{In this work,} we first define a novel fairness relaxation and show that it approximates the true fairness value better than existing relaxations. Second, we propose a model-agnostic gradient-based multi-objective algorithm that supports multiple sensitive attributes and all notions of fairness that require a form of statistical parity across groups. Experiments on four real-world datasets show that our novel relaxation integrated with the proposed multi-objective algorithm finds fair algorithms while suffering a lower loss of accuracy than state-of-the-art algorithms. Moreover, it performs effectively in simultaneously debiasing for multiple sensitive attributes and measures of fairness with a very low loss of utility.

\section{Related Work}
We consider the following notions of fairness for our analyses: \textit{demographic parity}~(DP) and \textit{equality of opportunity}~(EOP). Let the positive prediction be the favorable one in a binary classification problem. For example, for loan default prediction, predicting non-default is favorable. If the sensitive attribute is age with groups `young' and `adult,' DP requires the proportion of individuals labeled as positive to be the same for both `young' and `adult' groups. In contrast, EOP requires the true positive rate to be the same for both `young' and `adult' groups. These definitions are formalized in Section
~\ref{sec:background}.

\paragraph{Relaxation-based Approaches.} The approach used by \cite{donini2018empirical,zafar2017fpr,zafar2017linear} is to write DP or EOP in an equivalent but easier to handle form, and replace the indicator function by a relaxation.
\cite{zafar2017fpr,zafar2017linear} used a covariance measure between the sensitive attribute and the model parameters as a proxy for the fairness constraint. This leads to a convex constraint for DP~\citep{zafar2017linear} but a non-convex one for EOP~\citep{zafar2017fpr}. \cite{zafar2017fpr} proposed a convex-concave optimization process to deal with the non-convex constraint. For linear models, the covariance constraints reduce to a linear relaxation of the fairness measure. \cite{lohaus20} designed an elegant approach where they used an existing convex relaxation of the fairness measures as a regularization term in the loss function, with regularization parameter $\lambda$. They proved that the relaxed fairness constraint is a continuous function of $\lambda$, enabling a binary search of $\lambda$ to find a provably fair classifier. \cite{celis2019classification} proposed a method to solve multiple fairness measures simultaneously by reducing a constrained optimization of the loss function to an unconstrained problem by the lagrangian principle. 

While these methods are all attractive approaches and work well in practice for a single sensitive attribute, they suffer from two drawbacks:
\begin{enumerate*}
\item they cannot be integrated into any machine learning model, and \item require distinct and separate algorithms to solve.
\end{enumerate*}
Besides, \cite{lohaus20} require strong conditions on the classifier, and \cite{zafar2017fpr, zafar2017linear} cannot handle multiple fairness measures simultaneously. In comparison, our method handles multiple parity-based measures and is model-agnostic.

Moreover, several existing relaxations inadequately approximate the true fairness value: the relaxations might be satisfied, but the model may still be unfair~\citep{lohaus20}. Using the evaluation methods proposed in \cite{lohaus20} to gauge the effectiveness of different relaxations, we note that our novel relaxation is empirically better.

\paragraph{Multi-Objective Approaches.} The line of research that involves multiple objectives in fairness is very recent. \cite{valdivia2020fair} proposed an evolutionary approach to optimize for several objectives, using the multi-objective algorithm to search the space of hyperparameters of the model to find one that will work well on multiple objectives. However, it is possible that for some algorithms, there is no set of hyperparameters that perform well for all the objectives. This method is also infeasible to apply to large models since it involves training and evaluating hundreds of hyperparameter tuples. Finally, \cite{celis2019classification} proposed an algorithm for a class of statistical-based fairness measures. Their proposal is a meta-algorithm that operates by estimating conditional probabilities. In contrast, our approach can take any existing loss-based model as part of the multi-objective architecture. This would make it much easier for example to reuse production models which have already been optimized and would save the need to implement new architectures from scratch to account for fairness. 

\section{Background}
\label{sec:background}
Let $x \in \mathbb{R}^d$ be the features, where $d$ is the total number of features, and $x = (z, a_1, a_2 \ldots a_i \ldots a_t)$. Each $a_i$ refers to a \textit{sensitive attribute}, and $z$ the rest of the attributes. The feature space for $z$, $a_i$, and $x$ is denoted by $\mathcal{Z}$, $\mathcal{A}_i$, and $\mathcal{X}$, respectively. Therefore, the domain of $x$ is: 
\begin{equation}
   \mathcal{X}  = \mathcal{Z} \times \mathcal{A}_1 \times \mathcal{A}_2 \times \ldots \mathcal{A}_i \times \ldots \mathcal{A}_T
\end{equation} 
For the sake of simplicity of notation, we assume that we have only one sensitive attribute, that is $T=1$, and it is denoted simply by $a$, with feature space $\mathcal{A} = \{-1, 1\}$. Each individual is assigned an outcome $y$ from the feature space $\mathcal{Y} = \{-1, 1\}$, which is the label we want to predict for~$x$. Assume that there is a distribution $\mathcal{P}_{\mathcal{D}} $ over the domain $\mathcal{D} = \mathcal{X} \times \mathcal{A} \times \mathcal{Y}$. Each $(x, a, y)$ is sampled $i.i.d.$ from~$\mathcal{P}_{\mathcal{D}}$. We denote the predictor by $h:\mathcal{X} \rightarrow \mathcal{Y}$, where the predicted outcome of $x$ is $h(x) \in \{ -1, 1 \}$. We define $h(x)$ as $sign(f(x))$, where $f:\mathcal{X} \rightarrow \mathbb{R}$ maps each $x$ to a real-valued number, and is fair with respect to the sensitive attributes.

\paragraph{Demographic Parity (DP):} 
A classifier $f$ satisfies demographic parity if the probability of the outcome is independent of the value of the sensitive attribute:
\begin{equation}
\label{eq:dp}
    \Pr(f(x)>0|a = -1) = \Pr(f(x)>0 | a = 1)
\end{equation}

\paragraph{Difference of Demographic Parity (DDP):} The first step, in writing Equation~\ref{eq:dp} as an expression that can be used in a gradient-based optimization, is to relax the definition to be a difference between the expected values of quantities on either side of the equality. This is called the \textit{Difference of Demographic Parity} (DDP), defined as:
\begin{equation}
    \mathrm{DDP}(f) = \mathop{\mathbb{E}}_{\mathcal{P}_{\mathcal{D}}}[\mathbbm{1}_{f(x)>0} | a=- 1] - \mathop{\mathbb{E}}_{\mathcal{P}_{\mathcal{D}}}[\mathbbm{1}_{f(x)>0} | a=1]
\end{equation}
where $\mathbbm{1}_c$ is the indicator function on the condition $c$, which is to say that $\mathbbm{1}_c$ is $1$ if $c$ is true, and $0$ otherwise.

It is clear that when $\mathrm{DDP}(f) = 0$, we achieve perfect demographic parity, although that is usually not a realistic goal. We can relax this requirement by using a threshold: given a threshold $\tau \geq 0$, we say that $f$ is $\tau\text{-}\mathrm{DDP}$ fair if $|\mathrm{DDP}(f)| \leq \tau$. However, this is still not enough to define a differentiable relaxation; we only have an empirical estimate $\widehat{\mathcal{P}}_{\mathcal{D}}$ of $\mathcal{P}_{\mathcal{D}}$ consisting of $n$ points. In that manner, the empirical estimate of DDP can be written as:
\begin{equation}
\label{eq:ddp}
    \widehat{\mathrm{DDP}}(f) \; = \; \frac{1}{n_{-1}}\;\sum_{\substack{\widehat{\mathcal{P}}_{\mathcal{D}}\\ a=-1}} \mathbbm{1}_{f(x)>0} \; - \; \frac{1}{n_{1}}\;\sum_{\substack{\widehat{\mathcal{P}}_{\mathcal{D}}\\ a=1}} \mathbbm{1}_{f(x)>0}
\end{equation}
Here $n_{-1}$ is the number of points with $a=-1$ and $n_{1}$ is the number of points with $a=1$. The total number of points then is $n = n_{-1} + n_{1}$. This expression is very close to what we can use as a constraint. However, the main problem with using this expression directly in a gradient-based optimization is the non-differentiability because of the indicator function. The differences between different relaxations then come from how the indicator function is relaxed in the expression above. 

\paragraph{Equality of Opportunity (EOP):} A classifier $f$ satisfies equality of opportunity if the probability of getting a true positive is independent of the value of the sensitive attribute:
{\footnotesize
\begin{equation}
\label{eq:eop}
    \Pr(f(x)>0|a = -1, y=1) = \Pr(f(x)>0 | a = 1, y=1)
\end{equation}
}
\paragraph{Difference of Equality of Opportunity (DEO):}
We relax Equation~\ref{eq:eop}, similarly than for the demographic parity in Equation~\ref{eq:ddp}, to get the \textit{Difference of Equality of opportunity} (DEO). Then the empirical version of DEO is expressed as follows:
\begin{equation}
    \widehat{\mathrm{DEO}}(f) \; = \; \frac{1}{n_{-1}}\;\sum_{\substack{\widehat{\mathcal{P}}_{\mathcal{D}}\\ a=-1 \\ y=1}} \mathbbm{1}_{f(x)>0} \; - \; \frac{1}{n_{1}}\;\sum_{\substack{\widehat{\mathcal{P}}_{\mathcal{D}}\\ a=1 \\ y=1}} \mathbbm{1}_{f(x)>0}
\end{equation}

\subsection{Fairness Relaxations}
The differences between relaxations come from how the indicator function is relaxed in the expressions $\widehat{\mathrm{DDP}}(f)$ and $\widehat{\mathrm{DEO}}(f)$. We conduct all analyses for demographic parity; the extension to EOP is straightforward by conditioning on the positive label.

\paragraph{Linear Relaxations:} \cite{donini2018empirical,zafar2017linear} proposed a linear relaxation, where the indicator function is simply replaced by a linear classifier $f(x)$. $\widehat{\mathrm{DDP}}(f)$ can then be written in the following equivalent form after substituting $\mathbbm{1}_{f(x)>0}$ by $f(x)$~\citep{lohaus20}:
\begin{equation}
    \left|\;\mathrm{LR}_{\widehat{\mathrm{DDP}}}(f)\;\right|=\left|\;\frac{1}{n} \sum_{\widehat{\mathcal{P}}_{\mathcal{D}}}\; C\left(a, \widehat{\mathcal{P}}_{\mathcal{D}}\right) f(x)\;\right| \leq \tau
\end{equation}
where $C(a, \widehat{\mathcal{P}}_{\mathcal{D}})$ is simply a constant multiplicative factor.

\paragraph{Convex-Concave Relaxations:} \cite{zafar2017fpr} proposed the convex-concave relaxation, where $\mathbbm{1}_{f(x)>0}$ is relaxed to $\min(0, f(x))$. Let $\hat{p}_{1}$ be the empirical estimate of the proportion of individuals with $a=1$. For the case of such a relaxation for DDP, $\widehat{\mathrm{DDP}}(f)$ can be written in the following equivalent form after substituting $\mathbbm{1}_{f(x)>0}$ by $\min(0, f(x))$~\citep{lohaus20}:
{\small
\begin{equation}
\left|\mathrm{CCR}_{\widehat{DDP}}(f)\right|= \left|\frac{1}{n} \sum_{\widehat{\mathcal{P}}_{D}} \;C'\left(a, \widehat{\mathcal{P}}_{\mathcal{D}}\right) \min (0, f(x))\right| \leq \tau
\end{equation}
}

\begin{figure*}[t!] 
	\centering
	\begin{subfigure}[t]{0.242\textwidth}
		\centering
		\includegraphics[width=1\linewidth]{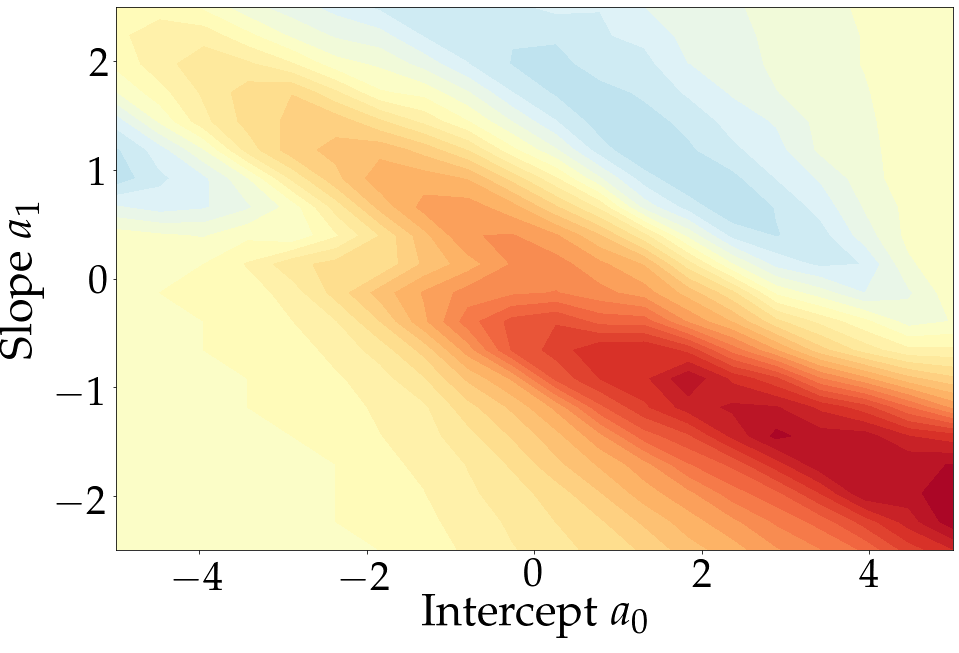}
		\caption{True DDP (Ideal)}\label{fig:ddp-true}		
	\end{subfigure}
	\begin{subfigure}[t]{0.23\textwidth}
		\centering
		\includegraphics[width=1\linewidth]{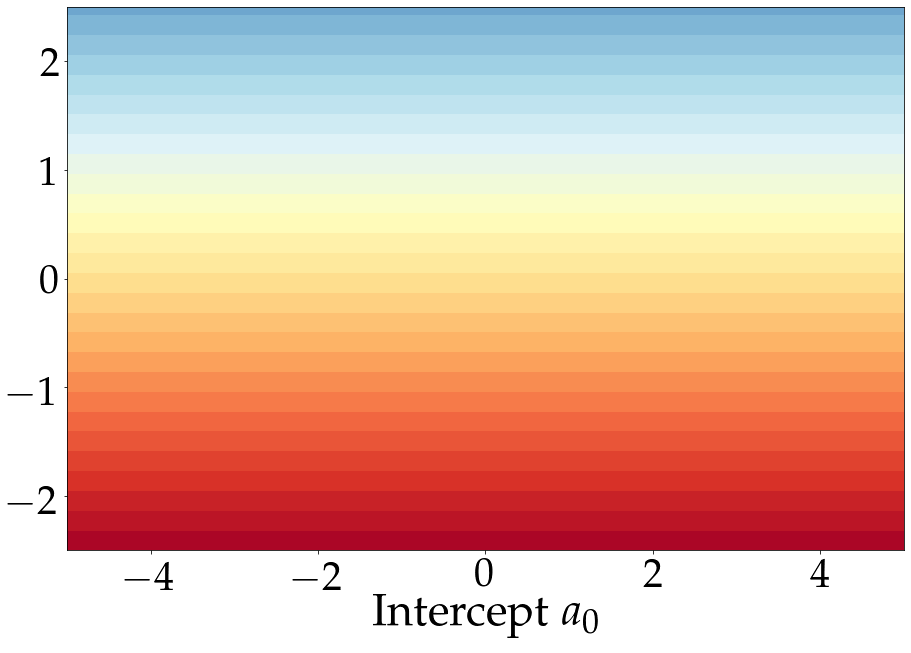}
		\caption{Linear}\label{fig:ddp-lr}
	\end{subfigure}
  \begin{subfigure}[t]{0.23\textwidth}
		\centering
		\includegraphics[width=1\linewidth]{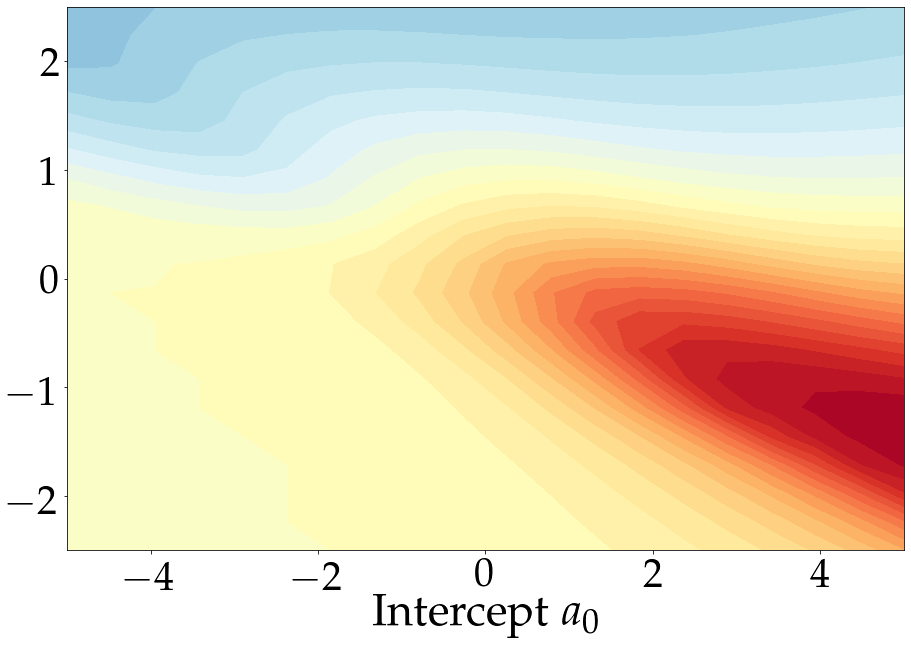}
		\caption{Convex-Concave}\label{fig:ddp-ccr}		
	\end{subfigure}
	\begin{subfigure}[t]{0.263\textwidth}
		\centering
		\includegraphics[width=1\linewidth]{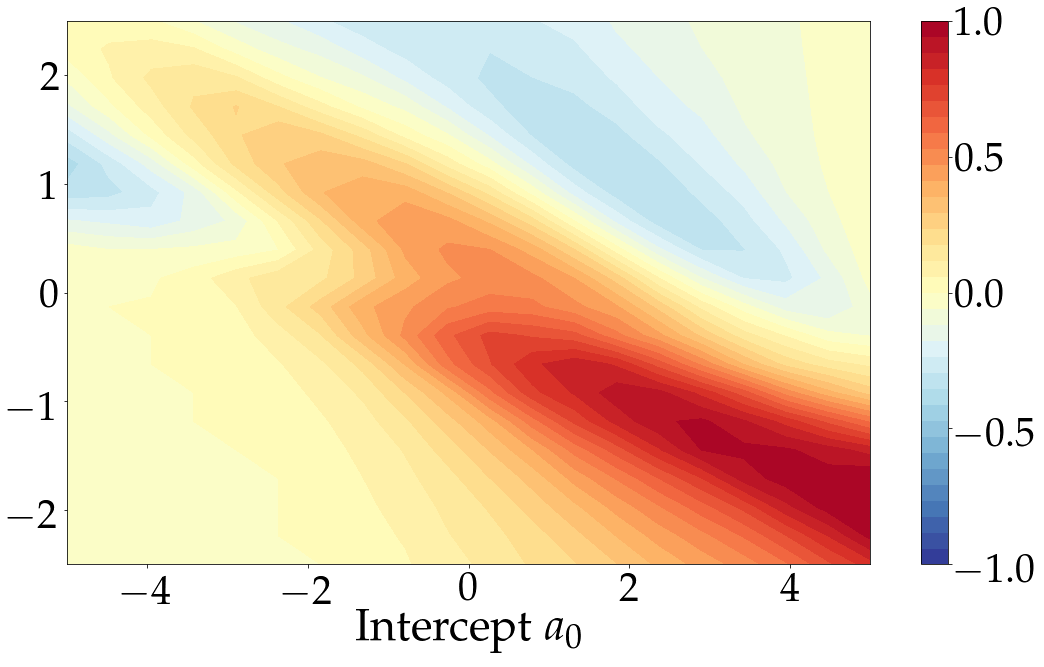}
		\caption{Hyperbolic Tangent (Ours)}\label{fig:ddp-tanh}
	\end{subfigure}
	\caption[Comparing fairness relaxations for DDP]{Each plot describes the family of linear classifiers in two dimensions which can be used to separate the classes in a two dimensional synthetic dataset. The decision boundary is defined as $x_2 = a_1x_1 + a_0$, meaning that $f(x)=-x_2 + a_1x_1 + a_0$. The point at $(a_0, a_1)$ on each plot gives normalized value of each of the following quantities for a classifier $f(x)$ with parameters $(a_0, a_1)$: (\ref{fig:ddp-true}) True Difference of demographic parity (DDP), (\ref{fig:ddp-lr}) Linear relaxation of the DDP, (\ref{fig:ddp-ccr}) Convex-concave relaxation of the DDP, and (\ref{fig:ddp-tanh}) Hyperbolic tangent relaxation (HTR). Yellow is fair. Ideally, we want the plot of the relaxations to be like that of the true DDP (\ref{fig:ddp-true}).}
\label{fig:compare-relax}
\end{figure*}
\section{A Novel Fairness Relaxation}
The existing relaxations described do not approximate the true DDP value accurately. To illustrate this, we use a two-dimensional toy dataset for binary classification, similarly to \cite{lohaus20}. Various Gaussian distributions are used to generate the points for each label. Each point is also assigned one of two groups to simulate the sensitive attribute. As we can see from Figure~\ref{fig:compare-relax}, existing relaxations do not faithfully capture the true DDP value.

To solve this problem, we introduce a new relaxation, called the hyperbolic tangent relaxation (HTR). Let $\sign(x)$ denote the signum of $x$, i.e. $sign(x)$ is $1$ if $x>0$, $-1$ if $x<0$ and $0$ if $x=0$. Figure~\ref{fig:compare-relax} further illustrates that our relaxation is the best at capturing the true DDP.

\begin{theorem}
\label{thm:tanh-sign}
The hyperbolic tangent of $n*x$ converges to the sign of $x$ for every fixed $x\in \mathbb{R}$ as $n$ goes to infinity. Formally,
\begin{equation}
\lim_{n\to \infty}\tanh(n*x)=\sign(x) \, \forall x\in \mathbb{R}
\end{equation}
\end{theorem}
\begin{proof}
The idea is that replacing $x$ by $n*x$ in $\tanh(x)$ compresses the horizontal scale. A more detailed proof is provided in the appendix .
\end{proof}
We can leverage Theorem~\ref{thm:tanh-sign} to find an expression that converges to the indicator function of $x>0$.
\begin{lemma}
$\tanh(n*\max(0, x))$ converges to the indicator function of $x>0$ as $n$ goes to infinity. Formally,
\begin{equation}
\lim_{n\to \infty}\tanh(n*\max(0, x))=\mathbbm{1}_{x>0} \, \forall x\in \mathbb{R}
\end{equation}
\end{lemma}
\begin{proof}
The proof is by simply replacing $x$ in Theorem~\ref{thm:tanh-sign} by $\max(0, x)$. The details are worked out in the appendix.
\end{proof}
\paragraph{Hyperbolic Tangent Relaxation (HTR):} Instead of relaxing $\mathbbm{1}_{f(x)>0}$ by $f(x)$ or $min(0, f(x))$ as proposed in the linear and convex-concave relaxations, respectively, we propose $\tanh(c*\max(0, f(x)))$, for small constants $c$. The larger the value of $c$, the better we can approximate the indicator function, but at the cost of degradation in the gradient's behavior. 

We denote $\tanh(c*\max(0, x))$ as $t(c, x)$. Formally, the hyperbolic tangent relaxation for the DDP, denoted by $HTR$ can be written as follows, for a chosen constant $c$:
{\small
\begin{equation} \label{eq:htr}
    \mathrm{HTR}_{\widehat{DDP}}(f) \; = \; \frac{1}{n_{-1}}\;\sum_{\substack{\widehat{\mathcal{P}}_{\mathcal{D}}\\ a=-1}} t(c, f(x)) \; - \; \frac{1}{n_{1}}\;\sum_{\substack{\widehat{\mathcal{P}}_{\mathcal{D}}\\ a=1}} t(c, f(x))
\end{equation}
}Finally, Figure~\ref{fig:compare-relax} demonstrates how the HTR is a better approximation of DDP than existing relaxations.
\section{The MAMO-fair Algorithm}
As our multi-objective optimization method, we use the algorithm of \cite{poirion2017descent} with modifications suggested by \cite{milojkovic2019multi}. We assume without loss of generality that all objectives are to be minimized. A multi-objective optimization problem can then be formulated as follows:
\begin{equation}
\label{eq:opt}
    \min\limits_{w \in \mathbb{R}^{d}} \; L(w) = \min\limits_{w \in \mathbb{R}^{d}} \; (\ell_1(w), \ell_2(w), \ldots \ell_k(w))
\end{equation}
where $\ell_i:\mathbb{R}^d \rightarrow \mathbb{R} \quad \forall i = 1, ..., k$ are the $k$ objectives, with $k \geq 2$. We interpret $L(w)$ as a multi-objective loss function and each $\ell_i(w)$ as one of the loss functions to be optimized by a machine learning model, with $w$ being the model parameters. Unlike in single-objective optimization problems, solutions of a multi-objective optimization problem are not ordered linearly. They are instead compared by dominance of solutions.

\begin{definition}[\textbf{Dominance of a Solution}]
A solution $w_1$ of Equation~\ref{eq:opt} dominates another solution $w_2 \ne w_1$ if $\ell_i(w_1)\leq \ell_i(w_2)\; \forall \; i = 1, ..., k$ and there exists $i_0 \in [1,k]$ such that $\ell_i(w_1) < \ell_i(w_2)$.
\end{definition}

\begin{definition}[\textbf{Pareto Optimality}]
A solution $w^*$ of Equation~\ref{eq:opt} is \textit{pareto optimal} if no other solution $w$ dominates~it. 
\end{definition}

\begin{definition}[\textbf{Pareto Front}]
The \textit{pareto front} of a set of solutions of Equation~\ref{eq:opt} is the set of all non-dominated solutions.
\end{definition}

We denote the gradient of objective $\ell_i(w)$ by $\nabla_w \ell_i(w)$. The key idea of the algorithm in optimizing simultaneously several objectives is to find a single vector, that gives the descent direction for every objective. This is called the common descent vector (CDV). The Karush-Kuhn-Tucker (KKT) conditions~\citep{karush1939minima,kuhn1951} provide necessary optimality conditions for the solution of a deterministic gradient-based optimization. A solution which satisfies the KKT conditions for a multi-objective optimization problem is called a pareto stationary point. 
\begin{definition}[\textbf{Pareto Stationary}]
A solution $w$ is pareto stationary if:
{\small
\begin{equation}
\exists (\alpha_1,, \alpha_2, \ldots,\alpha_k) \Bigg| \sum_{i=1}^{k} \alpha_{i}=1, \sum_{i=1}^{k} \alpha_{i} \nabla_{w} \ell_i(w)=0
\end{equation}
}
\end{definition}
Note that pareto stationarity is a necessary but not sufficient condition for optimality. The pareto stationary point admits a solution in the convex hull of the set $\{\nabla_{w} l_i(w) \; | \; i \in [k]\}$~\citep{desideri2012multiple}, which is the same as saying that the zero vector needs to be in the convex hull. The key idea is that the pareto stationary point can be found by iteratively solving the following optimization problem.
\begin{definition}[\textbf{Quadratic Constrainted Optimization Problem (QCOP)}]
The QCOP for our purpose is defined as follows:
\begin{equation}
\label{eq:min_sol}
    \min _{ \alpha_{1}, \ldots, \alpha_{n}}\left\{ \;\left\|\sum_{i=1}^{n} \alpha_{i} \nabla_{w} \ell_i(w)\right\|^{2} \Bigg| \sum_{i=1}^{n} \alpha_{i}=1, \alpha_{i} \geq 0 \; \right\}
\end{equation}
\end{definition}
Let $p^*$ be the vector of a solution of the Equation~\ref{eq:min_sol}, meaning that it is a convex combination of gradients specified by alpha. Then we have either:
\begin{enumerate}
    \item $\|p^*\|=0$, implies that the solution $w$ is pareto stationary;
    \item $\|p^*\|>0$, the solution $w$ is not pareto stationary and $\nabla_{w} \mathbf{L(w)} = p^*$, where $\nabla_{w} \mathbf{L(w)}$ denotes the common descent vector.
\end{enumerate}
The only key ingredient missing to describe the algorithm is the gradient normalization, proposed by \cite{milojkovic2019multi}. This allows us to overcome the issue of having losses with different scales.
\begin{definition}[\textbf{Gradient normalization}]
Let $l_i(w), \ldots, l_k(w)$ be the $k$ objectives and $\nabla_w(l_i(w))$ the gradient of $l_i(w)$ for all $i = 1, ..., k$. We define $w_{init}$ as the initial weight of the model. Then, we normalize the gradient as follows:
\begin{equation}
    \nabla_{w}\overline{l_i(w)} = \frac{\nabla_{w} l_i(w)}{l_i(w_{init})}
\end{equation}
\end{definition}
 We now have all the components to describe the final algorithm. The general idea is:
\begin{enumerate}
\item Calculate and normalize each gradient;
\item Find the common descent direction through QCOP;
\item Update gradients by performing the descent step;
\item Repeat for an appropriate number of batches and epochs.
\end{enumerate}
The pseudocode is provided in the appendix. The procedure is model-agnostic, so long as the model supports gradient-based optimization. In particular, unlike other methods which require convexity or are based on specific optimization algorithms, this method works well with neural networks as well. This is note-worthy because increasingly many real-world applications use complex non-convex models.

The key to using the algorithm is implementing fairness notions as loss functions, which is where our hyperbolic tangent relaxation comes into play.

\section{Experiments}

In this section, we assess the performance of our method based on experiments on four publicly available datasets.

\subsection{Datasets} 
We use the following datasets:
\begin{itemize}
\item \textbf{Adult}~\citep{Dua:2019}: the task is to predict if income is above or below 50k\$. Among the 14 features are attributes gender and race. We use \textit{sex} and a binarized version of \textit{race} as sensitive attributes. $y=1$ corresponds to the favorable prediction (income$\geq$50k\$). There are a total of 48,842 instances; 
\item \textbf{Compas}~\citep{propublica16}: the task is to predict if a defendant will racedeviate. There are 53 attribute, among them \textit{race} and \textit{sex}, which we use as sensitive attributes. There are 6,167 samples in total;
\item \textbf{Dutch census}~\citep{vzliobaite2011handling}: Census data of the Netherlands in 2001. Occupation is used as a proxy for low and high income, and \textit{sex} is used as a sensitive attribute. The data contains 60,420 instances with 12 features;
\item \textbf{Celeb attributes}~\citep{liu2015deep}: it is a dataset containing 202,599 face images of celebrities. This is accompanied by a list of 40 binary attributes for each image. We use this attribute dataset for classification, with the attribute \textit{smiling} used as a label, and \textit{sex} as a sensitive attribute. 
\end{itemize}
For the Compas dataset we use 3,000 samples for training, 2,000 for validation and the rest for testing. For the others we use 10,000 samples for training, 5,000 for validation, and the rest for testing.
\subsection{Baselines} 
\label{baselines}

\paragraph{Two Objectives.} We consider three baselines: a constrained optimization method with the linear relaxation of \cite{zafar2017linear}; the recent method of \cite{cotter2019two} for solving the lagrangian, and the searchFair algorithm of \cite{lohaus20}. We directly report the results of our baselines from \cite{lohaus20}. As the authors provide all experimental details necessary, we ensured to use precisely the same setting to be able to compare the relaxation-based approaches. In particular, we use the same sizes for training and test sets and the number of runs, as well as the same sets of features and pre-processing.

\paragraph{Beyond Two Objectives.} For more than two objectives, we cannot compare against traditional debiasing algorithms. In this case, we employ the following baselines:
\begin{itemize}
    \item \textbf{Sum of losses:} Multiple models with a single objective optimization. We represent the final objective as the sum of all objectives;
    \item \textbf{Unconstrained model:} A model without any constraint regarding fairness.
\end{itemize}

\subsection{Objectives} 

We recall that our algorithm solves the optimization problem described in Equation~\ref{eq:opt}. When optimizing for a single sensitive attribute for a single measure of fairness, we have two objectives: $\ell_1$ and $\ell_2$.
$\ell_1$ is the \textbf{performance objective}, for which we use the binary cross-entropy (BCE), and $\ell_2$ is the \textbf{fairness objective}. $\ell_2$ corresponds to the hyperbolic tangent relaxation of the fairness notion along with BCE added as a regularizer. For the DDP, the fairness objective is: 
\begin{equation}
\label{eq:fairness-objective}
    \ell_2 = \mathrm{HTR}_{\widehat{DDP}}(f) + \lambda*BCE(f)
\end{equation}
where $\lambda$ is the binary cross-entropy regularizer.
The regularizer is needed to avoid trivial constant solutions that attain perfect fairness, hence taking the fairness loss to zero.

\paragraph{The choice of the constant c in the HTR relaxation.} Recall that the HTR relaxation defined in Equation~\ref{eq:htr} requires the specification of a constant $c$ which allows us to decide how closely we want to approximate the true fairness. There is a trade-off between the behaviour of the gradient and the value of $c$. A higher value of $c$ gives a better approximation of the fairness value but a worse-behaving gradient. We choose $c=3$ based on the empirical results. Exploring the impact of this constant on the optimization process for various relaxations could be an interesting direction of future work.

For each additional sensitive attribute or fairness notion we want to optimize for, we add an analogous fairness objective. In other terms, the hyperbolic tangent relaxation of the fairness notion in question, with the BCE as a regularization term.

\subsection{Metrics}  
\label{metrics}

\paragraph{Single Fairness.} The goal is to learn classifiers that give the best improvement in fairness for the least decrease in accuracy, compared to the unconstrained model. We report the fairness difference metric (DDP or DEO) and the accuracy. We emphasize that DDP and DEO are representative choices, and the algorithm supports an array statistical parity based metrics. See Appendix~\ref{app:metrics} for details.

\paragraph{Multi-Fairness.} When having more than one sensitive attribute and/or fairness notion, a single point solution is not representative of the overall performance. Therefore, we compare the pareto fronts instead, that we denote by $S$. The pareto front consists of a set of points in $\mathbb{R}^k$, where $k$ is the number of objectives.

\paragraph{Constructing the pareto front.} The pareto front is constructed through a single training run of the algorithm. After each epoch, the trained model is added to the pareto front if it is pareto-optimal with respect to every point in the existing pareto front. The same method is used for the multi-objective algorithm. While doing several runs for our algorithm to construct the pareto front would make our results look stronger, we have avoided this to not give our approach an undue advantage over methods that do not have trade-off parameters.

As metrics, we employ the hypervolume and the spread of the pareto front:


\begin{itemize}
    \item \textbf{Hypervolume}~\citep{zitzler2007hypervolume}: the dominance volume enclosed by the pareto set in $\mathbb{R}^k$ with respect to the reference point. The larger the hypervolume, the better the solution. For our purpose, the reference point is always the origin;
    
    \item \textbf{Spacing}~\citep{okabe2003critical}: the spacing of~$S$ is a measure of how spread out the pareto front is. Spacing is low when the solutions are all in a single cluster, and high when they form a spread out pareto front. Formally, the spacing is defined as:
\begin{equation}
SP(S)=\sqrt{\frac{1}{|S-1|} \sum_{i=1}^{|S|}\left(d_{i}-\bar{d}\right)^{2}}
\end{equation}
where $d_i$ is the shortest $l_1$-norm from $s_i$ to any other point in $S$:
\begin{equation}
    d_{i}=\min _{s_{r} \in S, s_{r} \neq s_{i}} \sum_{m=1}^{k}\left|l_{m}\left(s_{i}\right)-l_{m}\left(s_{r}\right)\right|
\end{equation}
\end{itemize}

\subsection{Solution Selection} 
Selecting the best solution from the pareto front of a single run is nontrivial. \cite{wang2017application} list several strategies of selecting a point from the pareto front. Here we use the Linear Programming Technique for Multidimensional Analysis of Preference (LINMAP) method proposed by \cite{srinivasan1973linear}. LINMAP selects the point in the pareto front closest to ideal point. We choose this strategy as we can expect it to not favour a particular objective and give a model that finds a good trade-off between different objectives.
 
We use a training, validation, and test set for each run of the multi-objective algorithm. For each run, the model trained on the training set is evaluated on the validation set first, and the LINMAP strategy is used on the results of validation set to select the final point. The model corresponding to this point is the chosen model for each run and used for evaluation of the test samples. In this manner, we ensure that we are not fitting to the test samples for the results.
 
\begin{figure*}[ht] 
	\centering
	  \begin{subfigure}[t]{0.9\textwidth}
		\centering
		\includegraphics[width=1\linewidth]{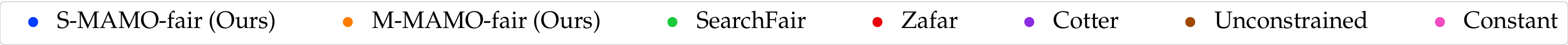}
	\end{subfigure}
\\
	\begin{subfigure}[t]{0.245\textwidth}
		\centering
		\includegraphics[width=1\linewidth]{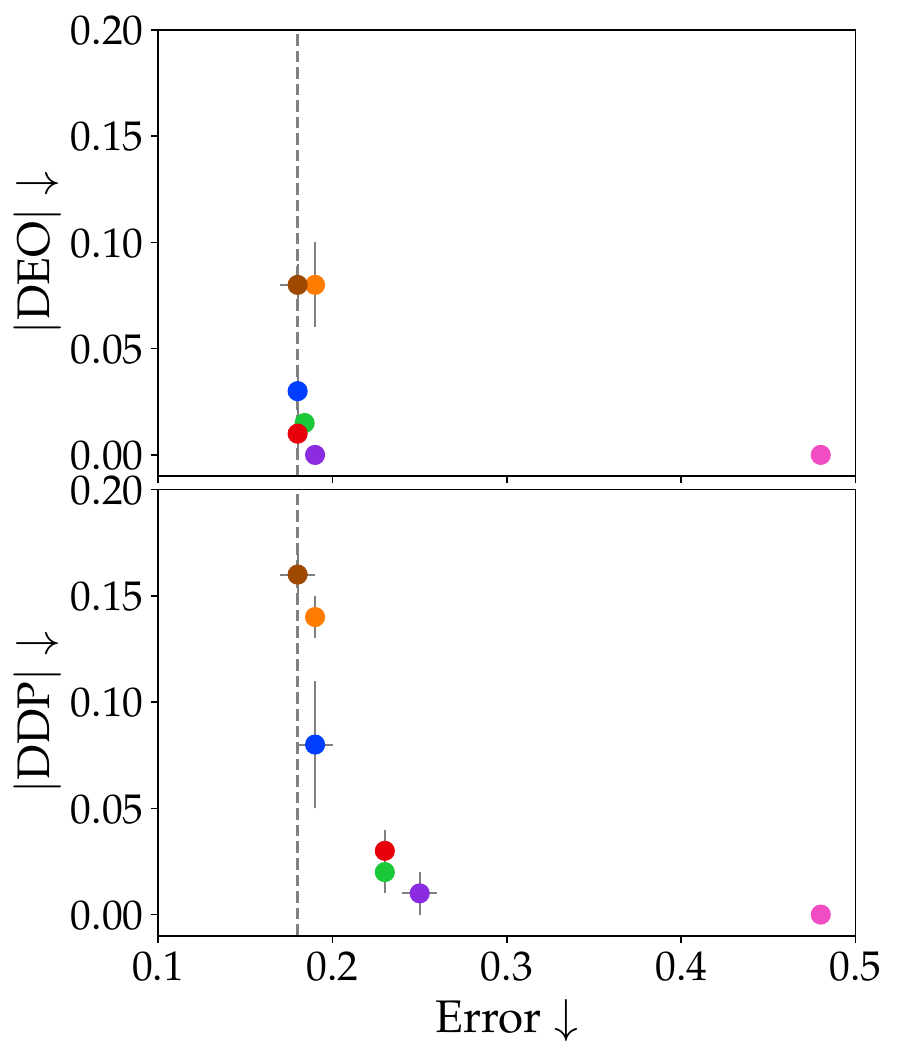}
		\caption{Dutch}
		\label{fig:dutch}
	\end{subfigure}
	\begin{subfigure}[t]{0.244\textwidth}
		\centering
		\includegraphics[width=1\linewidth]{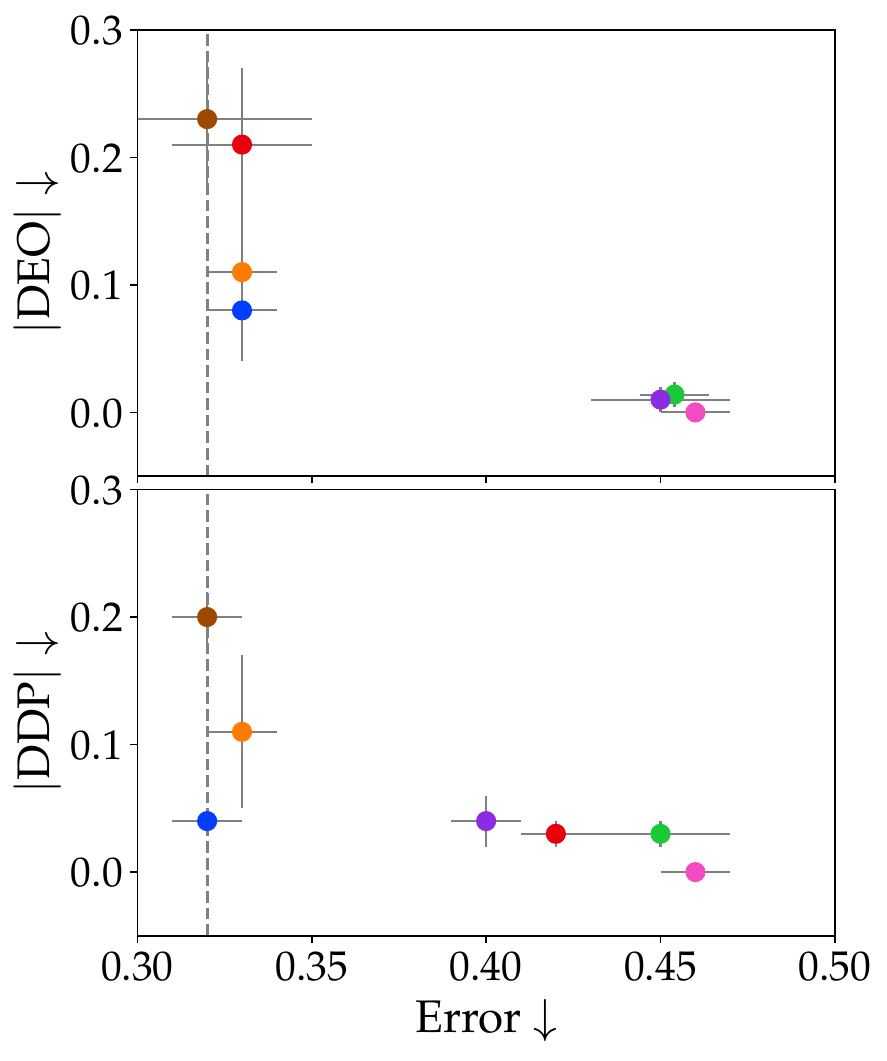}
		\caption{Compas}
		\label{fig:compas}
	\end{subfigure}
	\begin{subfigure}[t]{0.245\textwidth}
		\centering
		\includegraphics[width=1\linewidth]{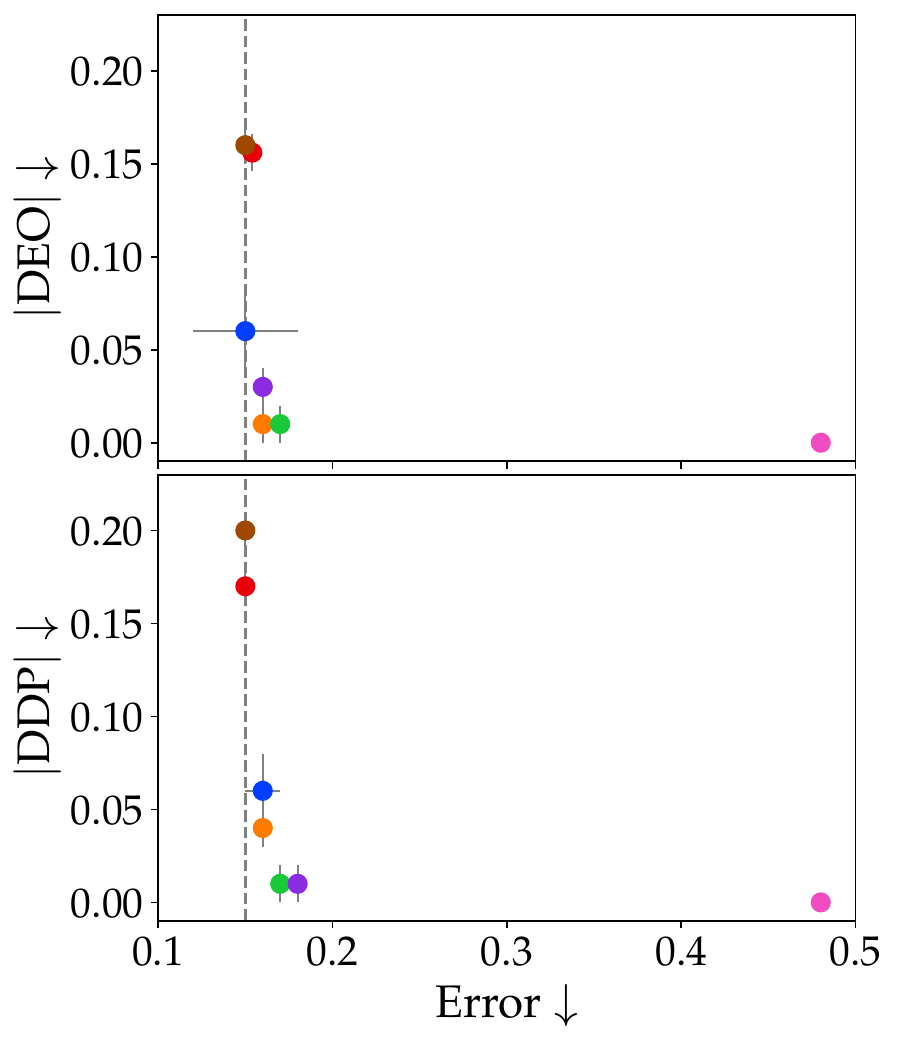}
		\caption{CelebA}
		\label{fig:celeb}
	\end{subfigure}
	\begin{subfigure}[t]{0.234\textwidth}
		\centering
		\includegraphics[width=1\linewidth]{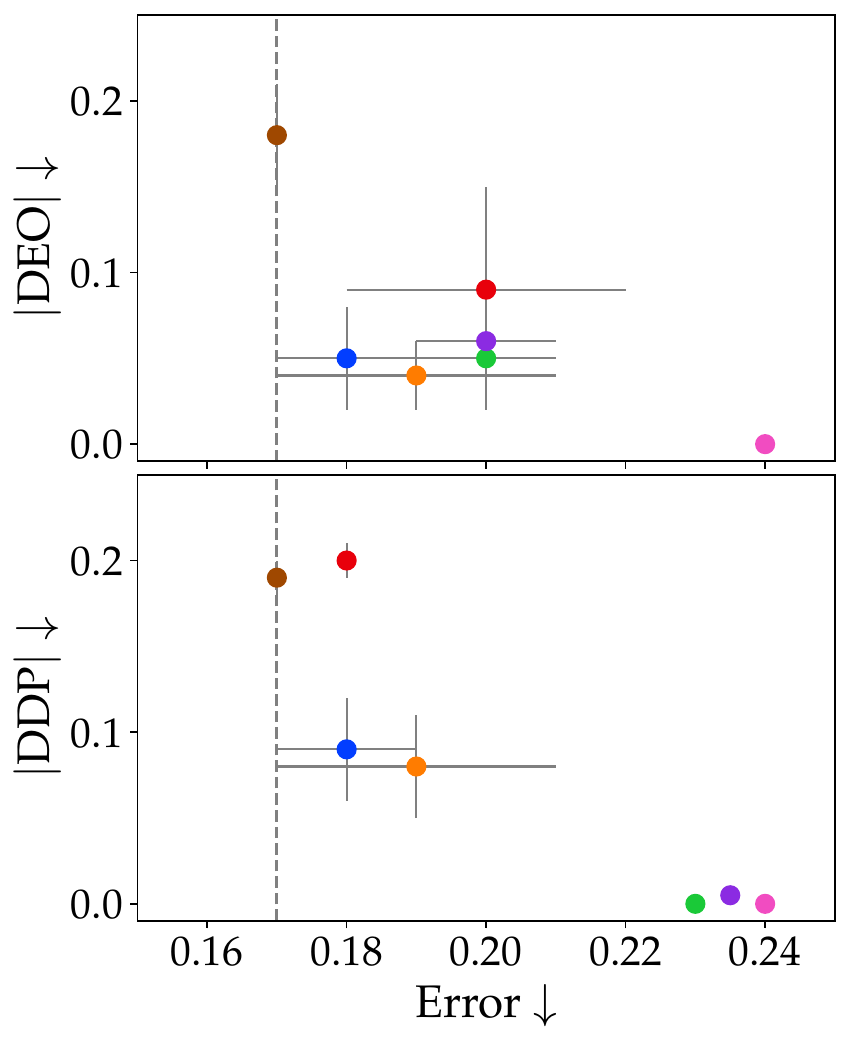}
		\caption{Adult}\label{fig:adult}
	\end{subfigure}
	\caption[Single-fairness results]{Single fairness results. The first and second rows contain the results for the DEO and DPP fairness metric, respectively. For both axes, a lower value is better. So the closer a point to the origin, the better the method. We emphasize that for each dataset, \textbf{M-MAMO-fair is optimized simultaneously for both DDP and DEO}. The dotted line shows the error of the unconstrained model. The closer a point is to the dotted line, the lower is the loss of accuracy suffered by the corresponding method. We see that our method achieves the best error among all methods while significantly improving fairness compared to the unconstrained model. The full tables are available in the appendix.}
\label{fig:compare-algos}
\end{figure*}

\subsection{Optimization Framework} 
We implemented the MAMO-fair algorithm as a publicly-available modular framework which implements most statistical parity based group fairness metrics. All implementation is in pytorch. The full list of implemented objectives is provided in the appendix. The framework is easy to extend by implementing other fairness notions and datasets, with instructions and documentation on how to do so provided with the implementation. This is in addition to the pre-processing and optimization code already available within the framework for the four datasets used in our experiments.

\subsection{Our Models}
We compare the baselines against two variants of our MAMO-fair model:
\begin{itemize}
    \item S-MAMO-fair (Single fairness): the algorithm optimizes for only one notion of fairness at a time; 
    \item M-MAMO-fair (Multi-fairness): the multi-fair MAMO-fair algorithm, where we have a single algorithm optimized simultaneously for DDP and DEO.\footnote{Code is available at https://github.com/kirtanp/MAMO-fair/}
\end{itemize}

One of the strengths of our approach is that it is model-agnostic, so it also works with neural networks unlike other debiasing algorithms~\citep{zafar2017fpr,zafar2017linear,celis2019classification,lohaus20}. We demonstrate it by using a feedforward neural network with 2 hidden layers of sizes 60 and 25 respectively, a ReLu activation function~\citep{xu2015empirical}, dropout~\citep{srivastava2014dropout} with $p=0.2$ between each layer, and a sigmoid at the output~layer.

\subsection{Hyperparameter Selection}
The most important hyperparameter choice is that of $\lambda$ in the fairness objectives (Equation~\ref{eq:fairness-objective}). We found that a value of $\lambda = 0.1$ works well for all datasets and both metrics. We use a batch size of 512 for Adult and Compas, and 200 for Dutch and CelebA. We use a learning rate of 0.01 for all experiments. We did not need to perform automated hyperparameter tuning of our method to achieve results comparable to the baselines. 

\section{Results}
We present the results of our experiments for single and multiple fairness objectives.

\subsection{Single Fairness}
Figure~\ref{fig:compare-algos} shows the results for the case of single-fairness. We see that our algorithm significantly improves on fairness with a very low loss of accuracy on both fairness metrics. While traditional models are optimized for a single fairness notion, we show that when trained on both fairness notions DDP and DEO simultaneously, our model (M-MAMO-fair) achieves higher performance on two out of four cases. 

\paragraph{Least Loss of Accuracy.} First, we see that our algorithm (S-MAMO-fair) always has the least loss of accuracy among all the methods. Second, we observe that whenever another algorithm matches the accuracy achieved by S-MAMO-fair, our model achieves a better performance on fairness. The only exception to this among the eight experiments is in Figure~\ref{fig:dutch} (DEO), where Zafar performs marginally better than the S-MAMO-fair algorithm with the same loss of accuracy. In Figure~\ref{fig:celeb} Zafar has a slightly better accuracy than our methods, but with a much worse fairness value.
\paragraph{Good Trade-off between Error and Fairness.} Methods that have a better performance than S-MAMO-fair on fairness often lose out significantly in the accuracy and end up being close to the trivial constant model. This is most clearly seen in results for the Adult and Compas datasets. For the Dutch and CelebA datasets, all methods perform well on fairness, but S-MAMO-fair still achieves the best accuracy, suggesting that these datasets are easier to debias than Adult and Compas.
\paragraph{Multi-Fairness Works Well.} Interestingly, we note that our multi-fairness algorithm outperforms single-fairness baselines in half of the cases. In particular, for the Adult and CelebA datasets, the M-MAMO-fair algorithm performs very close to the S-MAMO-fair algorithm and gives a better accuracy and better fairness than the baseline methods.

\paragraph{Inherent Limitations of Multi-Fairness.} For the Compas dataset, M-MAMO-fair performs well for DEO but not for DDP, which is in line with the \textit{impossibility results for fairness}: it is not possible to satisfy DP and error rate based metrics simultaneously if the base rate of classification is different for different groups~\citep{corbett2017algorithmic,goel2018non}. This explains the poor performance of the M-MAMO-fair algorithm on the Dutch dataset as well as the fact that it performs well only on DEO and not on DDP for the Compas dataset. However, this makes multi-objective algorithms for fairness even more essential, so as to find the best possible trade-offs between different fairness metrics, which our algorithm is shown to do well. The parameter $\lambda$ in the fairness objective (Equation \ref{eq:fairness-objective}) can be used to control the trade-off.
\begin{table}[t]
\centering
\caption[Multiple sensitive attributes for Compas]{Compas: simultaneously for race and gender.}
\label{ref:compas-multi-atr}
\begin{tabular}{@{}lccc@{}}
\toprule
& M-MAMO-fair & Sum of Losses & Unconstrained \\
\midrule
HV & \textbf{0.61 $\pm$ 0.01} & 0.55 $\pm$ 0.06 & 0.34 $\pm$ 0.03  \\
SP & \textbf{0.21 $\pm$ 0.09} & 0.05 $\pm$ 0.04 & 0.02 $\pm$ 0.01\\
\bottomrule
\end{tabular}
\end{table}
\subsection{Multi-Fairness}
Here we further illustrate the power of the algorithm to debias simultaneously for multiple sensitive attributes. Two of the datasets, Compas and Adult, contain both race and gender as sensitive attributes. For each dataset, we debias with respect to demographic parity simultaneously for race and gender. The metrics and baselines are as defined in Section \ref{metrics} and Section \ref{baselines} respectively. Table~\ref{ref:compas-multi-atr} and Table~\ref{ref:adult-multi-atr} show that our method outperforms the baselines on both metrics.

\begin{table}[t]
\caption[Multiple sensitive attributes for Adult]{Adult: simultaneously for race and gender.}
\label{ref:adult-multi-atr}
\centering
\begin{tabular}{@{}lccc@{}}
\toprule
& M-MAMO-fair & Sum of Losses & Unconstrained \\
\midrule
HV & \textbf{0.60 $\pm$ 0.09} & 0.30 $\pm$ 0.02 & 0.31 $\pm$ 0.05  \\
SP & \textbf{0.04 $\pm$ 0.02} & 0.02 $\pm$ 0.01 & 0.02 $\pm$ 0.01\\
\bottomrule
\end{tabular}
\end{table}

\section{Conclusion}
In this paper, we addressed the important problem of social discrimination in machine learning classifiers. We considered a specific class of debiasing algorithms which looks at relaxations of fairness notions. We have empirically shown that existing relaxations do not approximate the true fairness value well enough. 

Motivated by this, we proposed new relaxations which provably approximate fairness notions better than existing ones. In addition, we observed that debiasing is a naturally multi-objective problem, but there is a dearth of research in the field of multi-objective debiasing algorithms. We have taken a first step towards alleviating this scarcity by proposing a model-agnostic multi-objective method for finding fair and accurate classifiers. We demonstrated through experiments on four real-world publicly available datasets that our algorithm performs better than current state-of-the-art models at finding trade-offs between accuracy and fairness. Moreover, it can be used to simultaneously debias for multiple definitions of fairness and multiple sensitive attributes.

\newpage
\clearpage
\appendix

\begin{table*}[t]
\centering
\caption{\textbf{Results Table}: MF1 is the MAMO-fair algorithm optimizing separately for DEO and DDP, and MF2 is the algorithm optimizing simultaneously for DDP and DEO. SFa is the SearchFair algorith, Zaf is Zafar, Cot is Cotter, Unc is the unconstrained model and Con is the constant model}
\label{tab:1}
\begin{tabular}{@{}lcccc|cccc@{}}
\toprule
& \multicolumn{4}{c}{Adult} & \multicolumn{4}{c}{Compas} \\
\cmidrule(lr){2-5} \cmidrule(lr){6-9}
& \multicolumn{2}{c}{\textbf{Demographic parity}}  & \multicolumn{2}{c}{\textbf{Equality of opportunity}} & \multicolumn{2}{c}{\textbf{Demographic parity}}  & \multicolumn{2}{c}{\textbf{Equality of opportunity}} \\
\cmidrule(lr){2-3} \cmidrule(lr){4-5} \cmidrule(lr){6-7} \cmidrule(lr){8-9}
& $|\textrm{DDP}|$ & Error & $|\textrm{DEO}|$ & Error & $|\textrm{DDP}|$ & Error & $|\textrm{DEO}|$ & Error  \\
\midrule
MF1 &  0.09 $\pm$ 0.03 & \textbf{0.18 $\pm$ 0.01}   & 0.05 $\pm$ 0.03 & \textbf{0.18 $\pm$ 0.01}  & 0.04 $\pm$ 0.01 & \textbf{0.32 $\pm$ 0.01} & 0.08 $\pm$ 0.04 & \textbf{0.33 $\pm$ 0.01} \\
MF2 & 0.08 $\pm$ 0.03  & 0.19 $\pm$ 0.02   & 0.04 $\pm$ 0.02  & 0.19 $\pm$ 0.02  & 0.11 $\pm$ 0.06  & 0.33 $\pm$ 0.01 & 0.11 $\pm$ 0.07  & 0.33 $\pm$ 0.01  \\
\midrule
SFa & 0.00 $\pm$ 0.00 & 0.24 $\pm$ 0.00 & 0.05 $\pm$ 0.03 & 0.20 $\pm$ 0.01 & \textbf{0.03 $\pm$ 0.01} & 0.45 $\pm$ 0.02 & 0.01 $\pm$ 0.01 & 0.45 $\pm$ 0.01 \\  
Zaf & 0.20 $\pm$ 0.01 & 0.18 $\pm$ 0.00 &  0.09 $\pm$ 0.06 & 0.20 $\pm$ 0.02 & \textbf{0.03 $\pm$ 0.01} & 0.42 $\pm$ 0.01 & 0.21 $\pm$ 0.06 & 0.33 $\pm$ 0.02  \\
Cot & 0.00 $\pm$ 0.00 & 0.24 $\pm$ 0.00 & 0.06 $\pm$ 0.04 & 0.20 $\pm$ 0.01 & 0.04 $\pm$ 0.02 & 0.40 $\pm$ 0.01 & 0.01 $\pm$ 0.01 & 0.45 $\pm$ 0.02 \\
Unc & 0.19 $\pm$ 0.01 & 0.17 $\pm$ 0.00  & 0.18 $\pm$ 0.03 & 0.17 $\pm$ 0.00  & 0.20 $\pm$ 0.02 & 0.32 $\pm$ 0.01 & 0.23 $\pm$ 0.05 & 0.32 $\pm$ 0.03  \\
Con & 0.00 $\pm$ 0.00 & 0.24 $\pm$ 0.00 &  0.00 $\pm$ 0.00 & 0.24 $\pm$ 0.00 & 0.00 $\pm$ 0.00 & 0.46 $\pm$ 0.01 & 0.00 $\pm$ 0.00 & 0.46 $\pm$ 0.01 \\
\bottomrule
\end{tabular}
\end{table*} 

\begin{table*}[t]
\centering
\caption{\textbf{Results Table}: MF1 is the MAMO-fair algorithm optimizing separately for DEO and DDP, and MF2 is the algorithm optimizing simultaneously for DDP and DEO. SFa is the SearchFair algorith, Zaf is Zafar, Cot is Cotter, Unc is the unconstrained model and Con is the constant model}
\label{tab:2}
\begin{tabular}{@{}lcccc|cccc@{}}
\toprule
& \multicolumn{4}{c}{Dutch} & \multicolumn{4}{c}{CelebA} \\
\cmidrule(lr){2-5} \cmidrule(lr){6-9}
& \multicolumn{2}{c}{\textbf{Demographic parity}}  & \multicolumn{2}{c}{\textbf{Equality of opportunity}} & \multicolumn{2}{c}{\textbf{Demographic parity}}  & \multicolumn{2}{c}{\textbf{Equality of opportunity}} \\
\cmidrule(lr){2-3} \cmidrule(lr){4-5} \cmidrule(lr){6-7} \cmidrule(lr){8-9}
& $|\textrm{DDP}|$ & Error & $|\textrm{DEO}|$ & Error & $|\textrm{DDP}|$ & Error & $|\textrm{DEO}|$ & Error  \\
\midrule
MF1 &  0.08 $\pm$ 0.03 & \textbf{0.19 $\pm$ 0.01}  & 0.03 $\pm$ 0.01 & \textbf{0.18 $\pm$ 0.00}  & 0.06 $\pm$ 0.02 & 0.16 $\pm$ 0.01 & 0.06 $\pm$ 0.02  & \textbf{0.15 $\pm$ 0.03} \\
MF2 & 0.14 $\pm$ 0.01   & 0.19 $\pm$ 0.00   & 0.08 $\pm$ 0.02   & 0.19 $\pm$ 0.00   & 0.04 $\pm$ 0.01  & 0.16 $\pm$ 0.00 & 0.01 $\pm$ 0.01  & 0.16 $\pm$ 0.00  \\
\midrule
SFa & \textbf{0.02 $\pm$ 0.01} & 0.23 $\pm$ 0.00 & \textbf{0.01 $\pm$ 0.00} & 0.18 $\pm$ 0.00  & 0.01 $\pm$ 0.01 & 0.17 $\pm$ 0.00 & \textbf{0.01 $\pm$ 0.01} & 0.17 $\pm$ 0.00 \\  
Zaf & 0.03 $\pm$ 0.01 & 0.23 $\pm$ 0.00 & 0.01 $\pm$ 0.01 &  0.18 $\pm$ 0.00 & 0.17 $\pm$ 0.01 & 0.15 $\pm$ 0.00 & 0.16 $\pm$ 0.01 & 0.15 $\pm$ 0.00  \\
Cot & 0.01 $\pm$ 0.01 & 0.25 $\pm$ 0.01 & \textbf{0.00 $\pm$ 0.00} & 0.19 $\pm$ 0.00 & \textbf{0.01 $\pm$ 0.01} & 0.18 $\pm$ 0.00 & 0.03 $\pm$ 0.01 & 0.16 $\pm$ 0.00 \\
Unc & 0.16 $\pm$ 0.01   & 0.18 $\pm$ 0.01   & 0.08 $\pm$ 0.01   & 0.18 $\pm$ 0.01  & 0.20 $\pm$ 0.01 & 0.15 $\pm$ 0.00 & 0.16 $\pm$ 0.01 & 0.15 $\pm$ 0.00  \\
Con & 0.00 $\pm$ 0.00 & 0.48 $\pm$ 0.00 &  0.00 $\pm$ 0.00 & 0.48 $\pm$ 0.00 & 0.00 $\pm$ 0.00 & 0.48 $\pm$ 0.00 & 0.00 $\pm$ 0.00 & 0.48 $\pm$ 0.00 \\
\bottomrule
\end{tabular}
\end{table*}

\section{Toy dataset description}
Figure \ref{fig:toy-dta} provides a visualization of the toy dataset used in comparing the relaxations in Figure 1 in the main paper.

\begin{figure}[h]
\centerline{\includegraphics[scale=0.2]{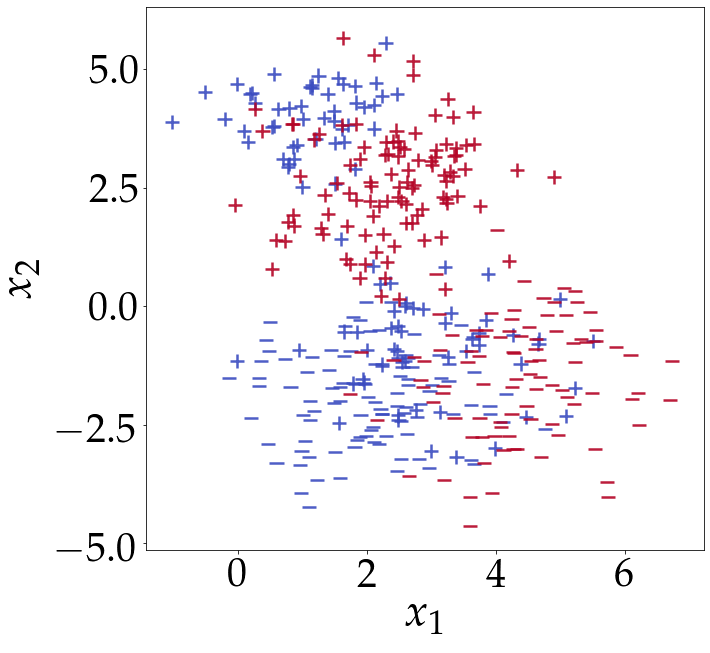}}
\caption[Toy Dataset]{A visualization of the toy dataset used in Figure 1 in the main paper. The class labels are $(+)$ and $(-)$. The color represents group membership for a binary sensitive attribute, so the two groups are \textit{red} and \textit{blue}. So the goal is to separate the class labels, and remain fair with respect to the colors. The dataset contains 600 points, but only 400 are shown for clarity.}
\label{fig:toy-dta}
\end{figure}
\paragraph{Dataset construction: } The dataset is taken directly from~\citep{lohaus20}. The points are drawn from various Gaussian distributions.
\begin{itemize}
    \item \textit{Protected sensitive attribute.}  We draw 150 points with a negative label from a Gaussian with mean $\mu_1 = [2, -1]$ and covariance $\Sigma_1 = [[1, 0], [0, 1]]$. For the positive label we draw 150 points from a mixture of two Gaussians, with $\mu_2 = [3, -1]$ and $\Sigma_2 = [[1, 0], [0, 1]]$ and $\mu_3 = [1, 4]$ and $\Sigma_3 = [[0.5, 0], [0, 0.5]]$.
    
    \item \textit{Unprotected sensitive attribute:} For the unprotected sensitive attribute,  we draw 150 points with a positive label from a Gaussian with mean $\mu_4 = [2.5, 2.5]$ and covariance $\Sigma_4 = [[1, 0], [0, 1]]$. For the negative label we draw 150 points from a Gaussian with $\mu_5 = [4.5, -1.5]$ and $\Sigma_5 = [[1, 0], [0, 1]]$.
\end{itemize}

\section{MAMO-fair algorithm}
Here we provide some further details on the multi-objective algorithm described in Section 5 of the main paper.

\begin{figure}[ht]
\centerline{\includegraphics[scale=0.35]{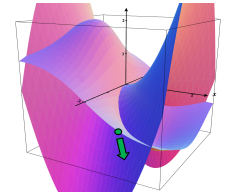}}
\caption[Common descent vector]{The figure gives an intuitive visualization of the common descent vector for two objectives. The two surfaces can be interpreted as loss functions for two objectives. The arrow points to the direction that minimizes both loss functions simultaneously.}
\label{fig:common-descent}
\end{figure}
\begin{algorithm}[ht]
        \caption{Final algorithm with gradient normalization}
        \label{alg:mamo-ap}
        \begin{algorithmic}[1] 
                \For{$i \in 1,...,k$}
                    \State $EL_i = \ell_i(w)$
                \EndFor
    
                \For{$epoch \in 1,...,M$}
                    \For{$batch \in 1,...,B$}
                        \State $forward\_ pass()$
                        \State $evaluate\_ model()$
                        \For{$i \in 1,...,n$}
                            \State $loss = \ell_i(w)$
                            \State $loss\_gradient = \nabla \ell_i(w)$ \vspace{2mm}
                            \State $\nabla\overline{\ell_i(w)} = \frac{\nabla_w\ell_i(w)}{EL_i}$
                        \EndFor
                        
                        \State $\alpha_1,...,\alpha_k = \textrm{QCOPSolver}\left(\nabla_w\overline{\ell_1(w)},...,\nabla_w\overline{\ell_k(w)}\right)$
                        \State $\nabla_wL(w)=\sum_{i=1}^k\alpha_i\nabla_w\overline{\ell_i(w)}$
                        \State $w = w-\eta\nabla_wL(w)$
                    \EndFor
                \EndFor
        \end{algorithmic}
    \end{algorithm}

Figure \ref{fig:common-descent} gives an intuition for a key ingredient of the multi-objective algorithm, the common descent vector. Algorithm \ref{alg:mamo-ap} provides the pseudocode for the algorithm.

\section{Supported Metrics}
\label{app:metrics}
Since the method is based on relaxing the indicator function, it supports all error-rate based metrics. We formally define some of them here. Table 1 in~\citep{celis2019classification} provides an even more complete list. The Figure~\ref{fig:app} defines metrics based on mis-classification rates of the prediction.
\begin{figure}[ht]
\centerline{\includegraphics[scale=0.35]{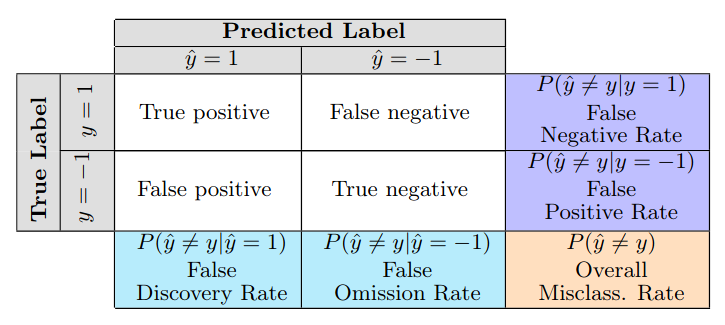}}
\caption[Mis-classification rate based fairness measures]{\label{fig:app}Table from \cite{zafar2017fpr} on disparate mistreatment based measures. The table defines the rates, the measure of fairness corresponding to each rate is the parity of that rate across groups}
\label{fig}
\end{figure}
We formally define some of the supported metrics next to give a general picture. 
\begin{definition}[\textbf{False Positive Rate}]
\label{def:fpr}
Parity of false positive rate
\[
    \Pr(\hat{y} = 1|a = - 1,\, y = - 1) = \Pr(\hat{y} = 1| a = 1,\, y = - 1)
\]
\end{definition}

\begin{definition}[\textbf{False Negative Rate}]
\label{def:fnr}
Parity of false negative rate across groups
\[
    \Pr(\hat{y} = - 1|a = - 1,\, y = 1) = \Pr(\hat{y} = - 1| a = 1,\, y = 1)
\]
\end{definition}

\begin{definition}[\textbf{True Positive Rate}]
\label{def:tpr}
Parity of true positive rates across groups
\[
    \Pr(\hat{y} = 1|a = - 1,\, y = 1) = \Pr(\hat{y} = 1| a = 1,\, y = 1)
\]
\end{definition}

\begin{definition}[\textbf{True Negative Rate}]
\label{def:tnr}
Parity of true positive rate across groups
\[
    \Pr(\hat{y} = - 1|a = - 1,\, y = - 1) = \Pr(\hat{y} = - 1| a = 1,\, y = - 1)
\]
\end{definition}

The relaxation procedure follows the same principle as described in the main content, where each fairness notion is written as a difference of expectation, further relaxed to an empirical estimate of the expectation. As a last step $\mathbbm{1}_{x>0}$ is relaxed to $\tanh(c*\max(0, x))$ and $\mathbbm{1}_{x<0}$ is relaxed to $\tanh(c*\min(0, x))$. Since we are using the relaxation to define a loss function for each measure of fairness, we have access to both $y$ and $\hat{y}$ while calculating the loss value. Therefore the method would also work for metrics such as the \textit{false discovery rate}, where we condition on the predicted value.

\section{Implementation details}
\paragraph{Design choice of the experiments:} The choice of using $10,000$ samples for training was to follow as closely as possible the experimental design of \cite{lohaus20} to make sure that the comparison is sound. They use this choice of sampling for their method and all the baselines, and we follow the same procedure. This choice means that we are using significantly less than $70\%$ of the samples for training, and all the rest for testing. The performance can only be expected to improve when we increase the training proportion and decrease the testing proportion of the dataset. This justifies the choice of $10,000$ samples  even in datasets with many more points (such as celebA).

\paragraph{Computational resources:} 
 All experiments were run on an ordinary laptop (16GB RAM and no GPU). Computational clusters were not used. 

\section{Result Tables}
Table \ref{tab:1} and Table \ref{tab:2} provide full tables for the results described in Figure 2 in the main paper. We note that in a few cases both the error and fairness value are identical for more than one baseline method. In this case we slightly perturb one of the values to ensure that all points are visible in the figure in the main paper. The tables in this appendix provide the values without this perturbation.

\section{Proof of Theorem 1}
Here we provide the proof of Theorem 1 from the main paper. First we give a reminder of the definition of the sign function
\begin{equation}
sign(x) =
  \begin{cases}
  1 & \text{if $x>0$} \\
  -1 & \text{if $x<0$} \\
  0 & \text{if $x=0$}
  \end{cases}    
  \label{eq:sign}
\end{equation}

\begin{observation}
The hyperbolic tangent is an odd function, which is to say that \[\tanh(-x) = -\tanh(x)\]
\label{obs:1}
\end{observation}

\begin{observation}[The quotient law of convergent series]
Let $(a_n)$ and $(b_n)$ be convergent series such that $\lim_{n\to \infty}a_n=A$ and $\lim_{n\to \infty}b_n=B$. Then we have
\[
\lim_{n\to \infty}\frac{a_n}{b_n}=\frac{\lim_{n\to \infty}a_n}{\lim_{n\to \infty}b_n} = \frac{A}{B}
\]
provided that $B\neq 0$.
\label{obs:2}
\end{observation}
Observation \ref{obs:2} is a commonly used result in real analysis. See Theorem C in~\citep{freiwald} for a proof.

\begin{theorem}
The hyperbolic tangent of $n*x$ converges to the sign of $x$ for every fixed $x\in \mathbb{R}$ as $n$ goes to infinity. Formally,
\begin{equation}
\lim_{n\to \infty}\tanh(nx)=\sign(x) \, \forall x\in \mathbb{R}
\end{equation}
\end{theorem}
\begin{proof}
We know from the definition of the hyperbolic tangent function that 
\begin{equation}
    \tanh(nx) = \frac{1-e^{-2nx}}{1+e^{-2nx}}
    \label{eq:def-tanh}
\end{equation}
The theorem requires pointwise convergence, meaning that the convergence in $n$ should hold for each value of $x$. Therefore $x$ can be though of as a constant for the purpose of the proof. Assuming $x$ to be a constant let $a_n = 1-e^{-2nx}$ and $b_n = 1+e^{-2nx}$. Then we have 
\begin{equation}
    \tanh(nx) = \frac{a_n}{b_n}
    \label{eq:tanh-ratio}
\end{equation}

We divide into cases by the value of $x$. 

\noindent \textbf{Case 1:} $\mathbf{x>0}$. In this case we have $\lim_{n\to \infty}e^{-2nx}=0$. Therefore it follows that $\lim_{n\to \infty}a_n=1$ and $\lim_{n\to \infty}b_n=1$. From Equation \ref{eq:tanh-ratio} we know that $\tanh(nx)$ is a ratio of $a_n$ and $b_n$. Therefore it follows from Observation \ref{obs:2} that 
\[
    \lim_{n\to \infty}\tanh(nx) = \frac{\lim_{n\to \infty}a_n}{\lim_{n\to \infty}b_n} = 1
\]

\noindent \textbf{Case 2:} $\mathbf{x<0}$. Since $x<0$, we have $-x>0$. Therefore from case 1 we know $\lim_{n\to \infty}\tanh(n(-x)) = 1$. We have from Observation \ref{obs:1} that $\tanh(-nx) = -\tanh(nx)$. Therefore, 
\[
\lim_{n\to \infty}\tanh(nx) = -\lim_{n\to \infty}\tanh(n(-x)) = -1
\]

\noindent \textbf{Case 3:} $\tanh(nx) = 0$ for $x=0$. Therefore
\[
    \lim_{n\to \infty}\tanh(nx) = 0
\]
Putting the three cases together we have
\[
  tanh(nx) =
  \begin{cases}
  1 & \text{if $x>0$} \\
  -1 & \text{if $x<0$} \\
  0 & \text{if $x=0$}
  \end{cases}
\]
This is identical to the definition of the sign function (Equation \ref{eq:sign}. Therefore,
\[
    \lim_{n\to \infty}\tanh(nx)=\sign(x) \, \forall x\in \mathbb{R}
\]
\end{proof}

\section{Proof of Lemma 1}
\begin{lemma}
$\tanh(n*\max(0, x))$ converges to the indicator function of $x>0$ as $n$ goes to infinity. Formally,
\begin{equation}
\lim_{n\to \infty}\tanh(n*\max(0, x))=\mathbbm{1}_{x>0} \, \forall x\in \mathbb{R}
\end{equation}
\end{lemma}
\begin{proof}
We know from Theorem~\ref{thm:tanh-sign} that 
\begin{equation}
    \lim_{n\to \infty}\tanh(n*\max(0, x))=\sign(\max(0, x))
\end{equation}
\paragraph{Case 1:}$x>0$. When $x>0$, $\max(0, x) = x$. Therefore we have $\sign(\max(0, x)) = \sign(x) = 1$. So $\sign(\max(0, x))=1$ when $x>0$.
\paragraph{Case 2:}$x\leq 0$. When $x\leq0$, $\max(0, x) = 0$ and therefore $\sign(\max(0, x))=0$.

So we have that $\sign(\max(0, x))=0$ for $x\leq 0$ and $\sign(\max(0, x))=1$ for $x>0$. But this is by definition the indicator function of $x>0$, $\mathbbm{1}_{x>0}$. Hence, $\sign(\max(0, x))=\mathbbm{1}_{x>0}$ and we can conclude that $\lim_{n\to \infty}\tanh(n*\max(0, x))=\mathbbm{1}_{x>0}$.
\end{proof}

\end{document}